\documentclass[tikz]{article}
\usepackage{proceed2e}
\usepackage{natbib}

\usepackage{times}
\usepackage{helvet}
\usepackage{courier}
\usepackage{url}
\usepackage{multirow}
\usepackage{pdfpages}

\usepackage[english, status=draft]{fixme}
\fxusetheme{color}
\newcommand{\fixred}[1]{\textcolor{red}{[#1]}}

\renewcommand{\cite}[1]{\citep{#1}}

\usepackage{amsmath}
\usepackage{amssymb}
\usepackage{amsthm}
\usepackage{mathtools}
\usepackage{blkarray}
\usepackage{url,array}

\usepackage{graphicx}
\usepackage{caption}
\usepackage{subcaption}
\usepackage{latexsym}

\usepackage{xstring}
\usepackage{pgfplots}
\usepackage{mathpazo}
\usepackage{marvosym,stackengine}
\usepackage{boolexpr}
\usepackage{xifthen}
\usepackage{url}
\usepackage{t1enc}
\usepackage{amsthm}
\usepackage[english]{babel}
\usepackage{url}

\usepackage{mathptmx}
\usepackage{fancyhdr}
\usepackage{amsmath, amsthm, amssymb, amsbsy}
\usepackage{bm}
\usepackage{pgfplots}
\usepackage{xparse}
\usepackage{xparse}

\usepackage{algorithm}
\usepackage{algpseudocode}

\newcommand\Algphase[1]{%
\vspace*{-.7\baselineskip}\Statex\hspace*{\dimexpr-\algorithmicindent-2pt\relax}\rule{0.5\textwidth}{0.4pt}%
\Statex\hspace*{-\algorithmicindent}\textbf{#1}%
\vspace*{-.7\baselineskip}\Statex\hspace*{\dimexpr-\algorithmicindent-2pt\relax}\rule{0.5\textwidth}{0.4pt}%
}

\usepackage{pifont}

\newtheorem{thm}{Theorem}

\newtheorem*{definition}{Definition}
\newtheorem{hypothesis}[thm]{Hypothesis}
\newtheorem*{remark}{Remark}

\newcounter{row}
\newcounter{col}

\newcommand\setrow[4]{
  \setcounter{col}{1}
  \foreach \n in {#1, #2, #3, #4} {
    \edef\x{\value{col} - 0.5}
    \edef\y{4.5 - \value{row}}
    \node[anchor=center] at (\x, \y) {\n};
    \stepcounter{col}
  }
  \stepcounter{row}
}

\newcommand\setroweight[8]{
  \setcounter{col}{1}
  \foreach \n in {#1, #2, #3, #4, #5, #6, #7, #8} {
    \edef\x{\value{col} - 0.5}
    \edef\y{8.5 - \value{row}}
    \node[anchor=center] at (\x, \y) {\n};
    \stepcounter{col}
  }
  \stepcounter{row}
}

\newcommand\setrownine[9]{
  \setcounter{col}{1}
  \foreach \n in {#1, #2, #3, #4, #5, #6, #7, #8,  #9} {
    \edef\x{\value{col} - 0.5}
    \edef\y{9.5 - \value{row}}
    \node[anchor=center] at (\x, \y) {\n};
    \stepcounter{col}
  }
  \stepcounter{row}
}

\tikzset{
  treenode/.style = {align=center, inner sep=0pt, text centered,
    font=\sffamily},
  arn_n/.style = {treenode, circle, white, font=\sffamily\bfseries, draw=black,
    fill=black, text width=1.5em},
  arn_r/.style = {treenode, circle, red, draw=red, 
    text width=1.5em, very thick},
  arn_x/.style = {treenode, rectangle, draw=fblack,
    minimum width=0.5em, minimum height=0.5em}
}

\tikzset{>=latex}

\pgfplotsset{every axis/.append style={
                    legend style={font=\tiny,line width=.5pt,mark size=.2pt},
                    }}

\newcommand\mygrid[3][]{      
  \let\mymatrixcontent\empty  
  \newcommand{\row}{%
    \foreach \j in {1,...,#2}{
      \foreach \i in {1,...,#3} {%
        \begingroup\edef\x{\endgroup
           \noexpand\gappto\noexpand\mymatrixcontent{|[draw,minimum size=1cm,#1]| \&}}\x
        }%
      \gappto\mymatrixcontent{\\}%
    }
  }
  \row
  \matrix(m)[matrix of nodes,ampersand replacement=\&,row sep=-\pgflinewidth,column sep=-\pgflinewidth]{
    \mymatrixcontent
  };
  \foreach \x[count=\i from 0] in {1,...,#2}\node[left] at (m-\x-1.west) {$\i$};
  \foreach \y[count=\j from 0] in {1,...,#3}\node[above] at (m-1-\y.north) {$\j$};
}

\pgfmathdeclarefunction{invgauss}{2}{%
  \pgfmathparse{sqrt(-2*ln(#1))*cos(deg(2*pi*#2))}%
}

\begin{document}

\title{Searching for Topological Symmetry in Data Haystack}

\author{Kallol Roy \qquad Anh Tong \qquad Jaesik Choi\\
Ulsan National Institute of Science and Technology\\
Ulsan, Korea\\
\{kallol,anhth,jaesik\}unist.ac.kr
 }
\maketitle
%
\begin{abstract}
Finding interesting symmetrical topological structures in high-dimensional systems is an important problem in statistical machine learning. Limited amount of available high-dimensional data and its sensitivity to noise pose computational challenges to find symmetry. Our paper presents a new method to find local symmetries in a low-dimensional 2-D grid structure which is embedded in high-dimensional structure. To compute the symmetry in a grid structure, we introduce three legal grid moves (i) Commutation (ii) Cyclic Permutation (iii) Stabilization on  sets of local grid squares, grid blocks. The three grid moves are legal transformations as they preserve the  statistical distribution of hamming distances in each grid block. We propose and coin the term of grid symmetry of data on the 2-D data grid as the invariance of statistical distributions of hamming distance are preserved after a sequence of grid moves. We have computed and analyzed the grid symmetry of data on multivariate Gaussian distributions and Gamma distributions with noise.
\end{abstract}

\section{Introduction}
\label{sec:intro}

The current hurdle in big data is to develop machine learning representations which can extract meaningful features. The principle of symmetry plays a natural foundation in the development of such a learning representation by getting rid of unimportant variations, while making the important ones easy to detect. Exploiting symmetries reduces computational complexity and leads to the development of new generalizations of learning algorithms and  provides a new approach in deep symmetry networks \cite{NIPS2014_5424,BadrinarayananM15}. There is recent interest in exploiting cyclic symmetry in convolution neural network architectures\cite{2016arXiv160202660D,P115.045}. Encoding these properties in networks by using the \textit{transnational equivariance} allows the model for parameter budgeting efficiently.     

Searching for symmetry in high-dimensional objects under certain \textit{low complexity} constraints though possible is a computationally challenging task. Symmetry based machine learning are broadly classified as $(1)$ Exchangeable variable models \cite{NiepertD14} $(2)$ Deep Symmetry Networks $(3)$ Symmetry based semantic parsing \cite{PoonD09}. 

One way to solve this problem is to use a topology-preserving dimensionality reduction method, and then search for symmetrical structures.  High-dimensional models with low-dimensional structures of patterns or symmetry are ubiquitous. Extracting low-dimensional structures in high-dimensional models have widespread uses in various disciplines including neuroscience, economics, and genetics. Our work is inspired by the Noether's Theorem of unification symmetry and conservation in theoretical physics \cite{schwarzbach2010noether}.
This paper presents a novel method of searching symmetry on $2-$D grid space, where the Betti number, an important topological property in persistent homology, is computed (or efficiently approximated). We define three grid moves (i) Commutation (ii) Cyclic Permutation (iii) Stabilization on grid blocks, consisting of a finite number of local grid squares. Our algorithm finds symmetry in each grid block after a finite sequence of grid moves.\newline
We prove the upper bound of the Hamming distance $H(n)$ is bounded. We have used the metric of Hamming distance as measure of randomness in the search of symmetry. The randomness may come from the added noise in the signal data or inherently embedded in the data. Our proposed method of topological data processing is immune the to effect of noise for most cases and is used for searching the local symmetry.  Our method of estimating the upper bound of the Hamming distance $H(n)$ can be useful in detecting the \textit{phase change} in data, which have profound implications in security, finance, and other areas. 
 
The organization of the paper is as follows: Section~\ref{LD} gives a quick introduction to Low-Dimensional Topological Models, Subsection~\ref{GH} will introduce the newly construct of Grid Diagrams perspective,  Section~\ref{CS} will explain our method of searching symmetry,  Section~\ref{Ising} explain our newly proposed Ising model of data, Section~\ref{RWW}  explains the related work, Section~\ref{ERI} presents our experimental results and Section~\ref{Conc}  concludes the paper. We have used the following important notations in our paper
\subsection{Notations}
$(1)$  \quad  $\beta$ \quad\quad\quad\quad       Betti number\newline
$(2)$  \quad  $H$          \quad\quad\quad\quad  Hamming distance\newline
$(3)$  \quad  $g$          \quad\quad\quad\quad  Small square grid \newline   
$(4)$  \quad  $G$          \quad\quad\quad\quad  2D Grid  \newline 
$(5)$  \quad  $l$         \quad\quad\quad\quad   Dimension of small grid\newline
$(6)$  \quad  $\mathbb{H}$ \quad\quad\quad\quad  Hamiltonian \newline
$(7)$  \quad  $\sigma$    \quad\quad\quad\quad   Configuration  \newline
$(8)$  \quad  $J_{g_{1},g_{2}}$ \quad\quad\quad  Grid interaction parameter \newline    $(9)$  \quad  $\Gamma(l)$      \quad\quad\quad   Scaling invariant parameter \newline
$(10)$  $T_{1},T_{2},T_{3}$ \quad Commutation, Cyclic Permutation, Stabilization

\section{Low-Dimensional Topological Models}\label{LD}
Our Algorithm of finding the symmetry in Data uses the topological features to search the local symmetry. Our method is of general nature and features other than the topological ones can be extended in it.  We encode our Data space as a topological space, because of its high-dimensional features  (symmetry and connectivity) can be inferred from its low-dimensional, local representations as in  \cite{Chen10,Edelsbrunner07,Carlsson14}. 

\subsection{Topological Invariants on Data Manifolds}
Computing topological invariants in low-dimensional space is used for exploring the symmetry in data space in our paper. Homology groups are increasingly used in computing the invariants as their computations are more feasible and provide the important information about the shape of the object. The homology groups for the  $2{-}D$ object are computed efficiently by the digitization \cite{Evako:2006:TPC:1143472.1649118,Chen10,chen2004discrete}. Digital topology allows discretizing data object by integrating the geometric and topological constraints. The digital model of a $2-$dimensional continuous object is called a digital $2-$surface. The intrinsic topology of the object is used without referring to an embedding space. A set $D$ is defined as a $2-$cell if it is homeomorphic to a closed unit square, similarly, a set $D$ is a $1-$cell if it is homeomorphic to a closed unit segment and a set $D$ is a circle(or $1-$sphere) if it is homeomorphic to a unit circle. The interior and the boundary of an $n-$cell $D$, are denoted as $IntD$ and $\partial D$ with the following boundary condition 
\begin{equation}\label{Homology}
D = IntD \cup \partial D
\end{equation}
Intuitively we can visualize the boundary of a $1-$ cell has two endpoints, the boundary of a $2-$cell is a circle. For the sake of completeness, we define $0-$cell as a single point for which $\partial D = \emptyset$.  The following properties hold for digital topology\newline
$\bullet$ For a  circle  $C$  and a $1-$cell $D$  contained in $C$ the set 
$C-IntD$ is a $1-$cell.\newline
$\bullet$ If $1-$cells $C_{1}$ and $C_{2}$ are such that $C_{1} \cap C_{2} = \partial C_{1}\cap \partial C_{2} = v$ holds, then $C_{1} \cup C_{2} = E$ is a $1-$cell.\newline
$\bullet$ For $2-$cells $D_{1}$ and $D_{2}$ such that $D_{1} \cap D_{2} = \partial D_{1} \cap \partial D_{2} = C$ is a $1-$cell holds, then $D_{1}\cup D_{2} = B$ is a $2-$ cell.\newline
$\bullet$ An $(i+1)-$cell can be formed by two disjoint $i-$cells that are parallel\newline
$\bullet$ An $i-$cell and it's parallel move form an $(i+1)-$cell.\newline
We now formally define the \textit{Digital Surface} as\newline

\begin{definition}[Digital Surface] 
A digital surface is the set of surface points each of which has two adjacent components not in the surface in its neighborhood
\end{definition}

In $2-$D space,  algorithms to compute $Betti$ numbers are of complexity $O(n\log^{2} n)$  or $O(n\log^{3} n)$. Our paper use the properties of manifolds in $2-$D digital spaces for the computation of topological invariants. We formally define \textit{digital manifold} as\newline  

\begin{definition}[Digital Manifold] 
A connected subset $S$ in digital space $\sum$ is a $i-$D digital manifold if\newline
$\bullet$ Any two $i-$cells are $(i-1)$ connected in $S$\newline
$\bullet$ Every (i-1)cell in $S$ has only one or two parallel-moves in $S$\newline
$\bullet$ $S$ does not contain any $(i+1)-$cell
\end{definition}

We represent a compact $3-$dimensional manifold in $R^{3}$ by a surface. Then the homology group is expressed in terms of its boundary surface. The Betti numbers related to homology groups are used in topological classification. For a $k-$ manifold, homology group $H_{i}, \quad i = 0, \cdots, k$, indicates the number of holes in each $i-$skeleton of the manifold. For a topological space $M$, its homology groups, $H_{i}(M)$ are certain measures of $i-$dimensional holes in $M$.  

\begin{definition}[Betti number] 
The \textit{Betti number} $\beta$ is formally defined as the rank of the quotient group as \newline
\begin{equation}\label{Homology}
\beta = \text{rank}  H_{i}(M)
\end{equation}
\end{definition}
In our algorithm we use the statistical distribution of Betti number $\beta$ in each small grid square of length $l$ for computing of local symmetry.

\subsection{Grid Diagrams for Low-Dimensional Topology}
\label{GH}
A \textit{grid diagram} is defined as a two dimensional square grid such that each  square inside the grid is filled with symbols $x$, $o$ or is left blank, with the constraint such that every column and every row has exactly one $x$ and one $o$. The symbols $x$ and $o$ are abstract decorations that fill the small grid square. 

The \textit{grid  number} for the grid diagram is the number of columns (or rows). The grid diagram is associated with an equivalent knot by joining the $x$ and $o$ symbols in each column and row by a straight line with the convention of vertical lines crosses over the horizontal lines (as shown by the dotted red lines in Figure~\ref{knot}). These lines joining the symbols $x$ and $o$ form the strands of the knot and removing the grid give us the planar projection of the knot (trivial knot in our example) as shown in Figure~\ref{knot} \cite{Manolescu12,ozsvath15,Sarkar10}. Grid diagrams are extensively used recently because of the use the grids gives a combinatorial definition of knot Floer homology \cite{Sarkar_acombinatorial}.  
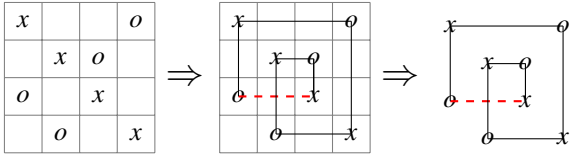
\begin{figure}[h!]
\centering
\begin{tikzpicture}
\draw[step=0.5cm,color=gray] (-1,-1) grid (1,1);
\node at (-0.75,+0.75) {$x$};
\node at (+0.75,+0.75) {$o$};
\node at (-0.25,+0.25) {$x$};
\node at (+0.25,+0.25) {$o$};
\node at (-0.75,-0.25) {$o$};
\node at (+0.25,-0.25) {$x$};
\node at (-0.25,-0.75) {$o$};
\node at (+0.75,-0.75) {$x$};
\node at (+1.40,0.00) {\Large $\Rightarrow$};
\end{tikzpicture} 
\begin{tikzpicture}
\draw[step=0.5cm,color=gray] (-1,-1) grid (1,1);
\node at (-0.75,+0.75) {$x$};
\node at (+0.75,+0.75) {$o$};
\node at (-0.25,+0.25) {$x$};
\node at (+0.25,+0.25) {$o$};
\node at (-0.75,-0.25) {$o$};
\node at (+0.25,-0.25) {$x$};
\node at (-0.25,-0.75) {$o$};
\node at (+0.75,-0.75) {$x$};
\draw (-0.75,+0.75) -- (+0.75,+0.75);
\draw (-0.25,+0.25) -- (+0.25,+0.25);
\draw[red,thick,dashed] (-0.75,-0.25) -- (+0.25,-0.25);
\draw (-0.25,-0.75) -- (+0.75,-0.75);
\draw (-0.75,+0.75) -- (-0.75,-0.25);
\draw (+0.75,+0.75) -- (+0.75,-0.75);
\draw (-0.25,+0.25) -- (-0.25,-0.75);
\draw (+0.25,+0.25) -- (+0.25,-0.25);
\draw (+0.25,+0.25) -- (+0.25,-0.25);
\node at (+1.40,0.00) {\Large $\Rightarrow$};
\end{tikzpicture} 
\begin{tikzpicture}
\node at (-0.75,+0.75) {$x$};
\node at (+0.75,+0.75) {$o$};
\node at (-0.25,+0.25) {$x$};
\node at (+0.25,+0.25) {$o$};
\node at (-0.75,-0.25) {$o$};
\node at (+0.25,-0.25) {$x$};
\node at (-0.25,-0.75) {$o$};
\node at (+0.75,-0.75) {$x$};
\draw (-0.75,+0.75) -- (+0.75,+0.75);
\draw (-0.25,+0.25) -- (+0.25,+0.25);
\draw[red,thick,dashed] (-0.75,-0.25) -- (+0.25,-0.25);
\draw (-0.25,-0.75) -- (+0.75,-0.75);
\draw (-0.75,+0.75) -- (-0.75,-0.25);
\draw (+0.75,+0.75) -- (+0.75,-0.75);
\draw (-0.25,+0.25) -- (-0.25,-0.75);
\draw (+0.25,+0.25) -- (+0.25,-0.25);
\draw (+0.25,+0.25) -- (+0.25,-0.25);
\end{tikzpicture} 
\caption{Knot Generation from Grid Diagram}
    \label{knot}
\end{figure}\newline
Three grid moves (explained in Section~\ref{CS}) to relate the grid diagrams are $(1)$ \textbf{Commutation $\mathbf{(T_{1})}$} $(2)$ \textbf{Cyclic Permutations $\mathbf{(T_{2})}$} $(3)$ \textbf{Stabilization $\mathbf{(T_{3})}$} \cite{Manolescu12,ozsvath15,Sarkar10}. These grid moves are analogous to Reidemeister moves for knot diagrams. The grid moves are used to generate equivalent relations. \textbf{Theorem}~\ref{q} explains that a sequence of grid moves gives the invariant knots.  A knot invariant is defined in the form of a polynomial such as the Alexander polynomial,  Conway polynomial, HOMFLY polynomial, Jones polynomial etc.

\begin{thm}{\cite{reidemeister32}}\label{q}
$G_{1}$ is a grid diagram with its equivalent knot $K_{1}$ and  grid diagram $G_{2}$ with its equivalent knot $K_{2}$. $K_{1}$ and $K_{2}$ are equivalent knots if and only if there exists a sequence of commutation, stabilization and cyclic permutation grid moves transform $G_{1}$ to $G_{2}$.
\end{thm}
Our goal in this paper is to define the symmetrical invariance in grid diagrams under uncertainty. For searching the local symmetry we have moved away from generating knots from the planar grid diagrams and instead use the distribution of hamming distance among the grid blocks explained in Section~\ref{CS}. A finite sequence of operations $T$ comprising of Commutation, Cyclic Permutation, Stabilization in a defined order is 
\begin{equation}\label{sequenceOp}
T = T_{1}^{a}\circ T_{2}^{b} \circ T_{3}^{c}
\end{equation}
where $a,b,c \in \mathcal{R}$. The stabilization operations $T_{3}$ of kink addition and kink subtraction occur in pairs to maintain the constant grid number.

\section{Searching Symmetry with Uncertainty}
\label{CS}
Inference in high dimensional data is challenge because of the curse of dimensionality. Thus, high dimensional data are usually converted to low-dimensional codes by $(1)$ Neural Networks \cite{Hinton06}; $(2)$ Nonlinear dimension reduction \cite{Tenenbaum00,lee07}; and $(3)$ Topological and Geometric methods \cite{wang12}. 

In this paper, we propose a new form of symmetry termed a \textbf{grid symmetry} with the following hypothesis.
\begin{hypothesis}[Invariance of Symmetric Probability]
The symmetry on a grid is represented by Commutation, Cyclic Permutation and Stabilization. The statistical distribution of the Betti numbers remains conserved during the above defined legal transformations. 
\end{hypothesis}

Symmetry of a geometric object comes with the concept of automorphisms. Legal transformations allowed for grid diagrams are $1$ \textbf{Commutation} $(2)$ \textbf{Cyclic Permutations} $(3)$ \textbf{Stabilization} defined as in \cite{ozsvath15,Manolescu12,Sarkar_acombinatorial,Sarkar10,Ozsvath04,Hedden08,reidemeister32}. \newline
\textbf{Commutation:} Commutation is defined as an interchange of two consecutive rows or columns of a grid diagram. The commutation is permitted only between rows or columns those are \textit{non-nested}.
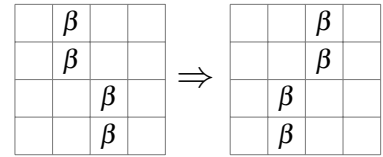
\begin{figure}[h]
\centering
\begin{tikzpicture}
\draw[step=0.5cm,color=gray] (-1,-1) grid (1,1);
\node at (-0.25,+0.75) {$\beta$};
\node at (-0.25,+0.25) {$\beta$};
\node at (+0.25,-0.25) {$\beta$};
\node at (+0.25,-0.75) {$\beta$};
\node at (+1.40,0.00) {\Large $\Rightarrow$};
\end{tikzpicture} 
\begin{tikzpicture}
\draw[step=0.5cm,color=gray] (-1,-1) grid (1,1);
\node at (+0.25,+0.75) {$\beta$};
\node at (+0.25,+0.25) {$\beta$};
\node at (-0.25,-0.25) {$\beta$};
\node at (-0.25,-0.75) {$\beta$};
\end{tikzpicture}
\caption{Commutation of a grid diagram}
    \label{Commutation of Grid Diagram}
\end{figure}
\\
\textbf{Cyclic Permutation:} Cyclic permutation preserves the grid number and is defined as the removal of an outer row/column and replacing it to the opposite side  of the grid.
\begin{figure}[h!]
\centering
\begin{tikzpicture}
\draw[step=0.5cm,color=gray] (-1,-1) grid (1,1);
\node at (-0.75,+0.75) {$\beta$};
\node at (+0.75,+0.75) {$\beta$};
\node at (-0.25,+0.25) {$\beta$};
\node at (+0.25,+0.25) {$\beta$};
\node at (-0.75,-0.25) {$\beta$};
\node at (+0.25,-0.25) {$\beta$};
\node at (-0.25,-0.75) {$\beta$};
\node at (+0.75,-0.75) {$\beta$};
\node at (+1.40,0.00) {\Large $\Rightarrow$};
\end{tikzpicture} 
\begin{tikzpicture}
\draw[step=0.5cm,color=gray] (-1,-1) grid (1,1);
\node at (+0.25,+0.75) {$\beta$};
\node at (+0.75,+0.75) {$\beta$};
\node at (-0.75,+0.25) {$\beta$};
\node at (-0.25,+0.25) {$\beta$};
\node at (+0.75,-0.25) {$\beta$};
\node at (-0.25,-0.25) {$\beta$};
\node at (-0.75,-0.75) {$\beta$};
\node at (+0.25,-0.75) {$\beta$};
\end{tikzpicture}
\caption{Cyclic Permutation of a grid diagram}
    \label{Cyclic Permutation of Grid Diagram}
\end{figure}
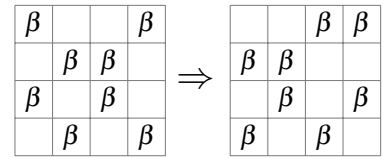
\\
\textbf{Stabilization:} Stabilization is performed by kink addition or removal and thus does not preserve the \textit{grid number}. A kink is added either to the right or left of a column or above or below of a row. Adding a kink to a column $c$ is done by inserting an empty row between the symbols $x$ and $o$ of the column $c$. Then an empty column is inserted either to the right or left of column $c$. We then move the either symbol $x$ or $o$ to the adjacent grid square in the added column. We then complete the added row and column with the symbols $x$ and $o$ appropriately. To add a kink to a row, we have to swap the row and column operations. To remove a kink (grid number decreases by $1$), we follow the instructions in reverse order. 
\begin{figure}
\centering
\begin{tikzpicture}[scale=0.45]
\begin{scope}
\draw[color=gray] (0,0) grid (6,6);
\setcounter{row}{1}
\setroweight { }{ }{ }{ }{ }{ }{ }{ }{ }
\setroweight { }{ }{ }{ }{ }{ }{ }{ }{ }
\setroweight { }{ }{ }{ }{ }{ }{ }{ }{ }
\setroweight { }{ }{ }{ }{ }{ }{ }{ }{ }
\setroweight { }{ }{$\beta$}{ }{$\beta$}{ }{ }{ }{ }
\end{scope}
\begin{scope}[xshift=8cm, yshift = -0.5cm]
\draw[color=gray] (0,0) grid (7,7);
\setcounter{row}{1}
\setrownine { }{ }{ }{ }{ }{ }{ }{ }{ } { }
\setrownine { }{ }{ }{ }{ }{ }{ }{ }{ } { }
\setrownine { }{ }{ }{ }{ }{ }{ }{ }{ } { }
\setrownine { }{ }{ }{ }{ }{ }{ }{ }{ } { }
\setrownine { }{ }{ }{ }{$\beta$}{$\beta$}{ }{ }{ } { }
\setrownine { }{ }{$\beta$}{ }{$\beta$ }{ }{ }{ }{ } { }
\end{scope}
\node at (+7 ,2.95) {\Large $\Rightarrow$};

\setcounter{row}{1}
\end{tikzpicture}
\caption{Stabilization of a grid diagram}
    \label{Stabilization of Grid Diagram}
\end{figure}
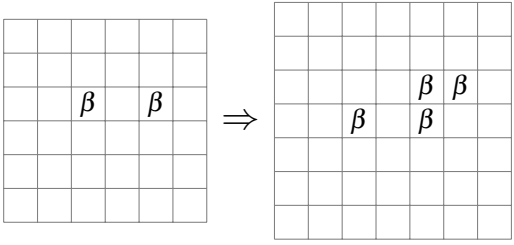

After projecting the high dimensional data to our $2$D grid space, 
we compute the Betti number $\beta$ in each small grid square of length $l$. We compute only a particular order of Betti number $\beta_{k}$ for a fixed $k$ and mark the grid square with the Betti number if $\beta_{k} \neq 0$ and leave the grid square empty if $\beta_{k} = 0$. For example we fill up the small grid square if the number of holes in it is at least $1$ i.e $\beta_{2} \neq 0$ and leave the grid square empty  when $\beta_{2} = 0$. We have used the symbol $\beta$ to represent $\beta_{k}$ for fixed $k$. We introduce this binary topological marking to get the sparse representation. For sufficiently sparse data we get a grid diagram where most of small grid squares are left empty and others are marked with $\beta$. This binary marking makes our model consistent with the grid homology for the studying of invariance of knots. Here, the difference is that we investigate the invariance of the probabilistic distribution of Betti numbers using the metric of Hamming distance.

\begin{figure*}[!h]
\begin{tikzpicture}[scale=.5]
	\begin{scope}
		\draw (0,0) grid (4,4);
        \node[anchor=center] at (2, -0.5) {Original};
        \setcounter{row}{1}
        \setrow {$\beta$}{$\beta$}{ }{ }
        \setrow {$\beta$}{ }{ }{ }
        \setrow { }{$\beta$}{ }{$\beta$}
        \setrow { }{$\beta$}{ }{$\beta$}
	\end{scope}
    
    \begin{scope}[xshift=7cm]
		\draw (0,0) grid (4,4);
        \node[anchor=center] at (2, -0.5) {$H = 12$};
        \setcounter{row}{1}
        \setrow { }{$\beta$}{$\beta$}{ }
        \setrow { }{$\beta$}{ }{ }
        \setrow {$\beta$}{ }{$\beta$}{ }
        \setrow {$\beta$}{ }{$\beta$}{ }
	\end{scope}
    
    \begin{scope}[xshift=14cm]
		\draw (0,0) grid (4,4);
        \node[anchor=center] at (2, -0.5) {$H = 6$};
        \setcounter{row}{1}
        \setrow { }{ }{$\beta$}{$\beta$}
        \setrow { }{ }{$\beta$}{ }
        \setrow { }{$\beta$}{ }{$\beta$}
        \setrow { }{$\beta$}{ }{$\beta$}
	\end{scope}
    
    \begin{scope}[xshift=21cm]
		\draw (0,0) grid (4,4);
        \node[anchor=center] at (2, -0.5) {$H = 14$};
        
        \setcounter{row}{1}
        \setrow { }{ }{$\beta$}{$\beta$}
        \setrow { }{ }{ }{$\beta$}
        \setrow {$\beta$}{ }{$\beta$}{ }
        \setrow {$\beta$}{ }{$\beta$}{ }
	\end{scope}
    
    \begin{scope}[xshift=28cm]
		\draw (0,0) grid (4,4);
        \node[anchor=center] at (2, -0.5) {$H = 6$};
        
        \setcounter{row}{1}
        \setrow { }{ }{$\beta$}{$\beta$}
        \setrow { }{ }{$\beta$}{ }
        \setrow { }{$\beta$}{ }{$\beta$}
        \setrow { }{$\beta$}{ }{$\beta$}
	\end{scope}
    
    \draw[->] ( 4.2, 2) -- (6.8,2) node[midway, above = 5pt] {$T_2$};
    \draw[->] ( 11.2, 2) -- (13.8,2) node[midway, above = 5pt] {$T_2$};
    \draw[->] ( 18.2, 2) -- (20.8,2) node[midway, above = 5pt] {$T_1$};
    \draw[->] ( 25.2, 2) -- (27.8,2) node[midway, above = 5pt] {$T_1$};
\end{tikzpicture}
\centering
\caption{An illustrative example of computing Hamming distance as result of a sequence of operations $T_2 \circ T_2 \circ T_1 \circ T_1$}
\label{Ilustration Hamming}
\end{figure*}

After the marking of the grid squares, we randomly sample a square grid block of size $l$ consisting of $n^{2}$ small grid squares and apply a finite sequence of operations $T$ defined in Equation~\ref{sequenceOp} on the sampled grid block as shown in Figure~\ref{Hamming}. The small grid squares colored red and green as illustrated in Figure~\ref{Hamming} are of dimension  $l$.  The sampled grid block $ABCD$ on which the sequence of operations $T$ are applied is shown as  shaded grey in Figure~\ref{Hamming}.  The sequence of finite operations changes the arrangements of Betti numbers in the sampled grid block denoted as $H$ in Figure~\ref{Hamming}. The grid squares which were not marked with Betti number $\beta$ before the transformation $T$ may now be marked or filled with Betti number $\beta$. This means to say that the position arrangements of Betti number $\beta$. 

We now introduce the concept of symmetry as the amount of reshuffling happen because of the application of Transformation operation $T$. We have used the metric of Hamming  distance $H_{x}$ to capture the degree of reshuffling. We have moved away from the elementary concept of mirror symmetry and introduced the concept of \textbf{Grid Symmetry}. Our grid symmetry is with the respect to a particular feature of the data like distribution of holes, connected components etc. We have particularly used the topological features as it is more robust to noise. Our method is quite general and can be extended to other features of the data.
To compute the Hamming distance, we have used the following notations\newline
\begin{enumerate}
\item Each small grid square at location $i$ and $j$ is marked as  $(i,j)$.
\item |(i,j)| represents the occupation of the small grid square with the Betti number ($\beta\neq 0$) and is defined formally as,
 	\begin{equation*}
	|(i,j)| = \left\{
	\begin{array}{ll}
1  & \mbox{if } \text{the grid square is occupied } \\
0 & \mbox{if } \text{the grid square is not occupied } 
\end{array}
\right.
\end{equation*}
\item The Hamming distance $H_{i,j}$ computed along each row is given by  
	\begin{equation*}
	H_{i,j} =  |(i,j)| \oplus |(i,j^{'})|,
	\end{equation*}
where the interchange of two consecutive columns $T = T_{1}$ is given by $\pi(j) = j^{'}$, and pasting the outermost column before the first $T = T_{2}$  is given $\pi(1) = j^{'}$ for. $\pi$ is the permutation operator. Here $T_{1}$ and $T_{2}$ are commutation and cyclic permutations as defined in the Section~\ref{LD}.
\item The Hamming distance computed over the sampled  grid block $H(n)$ is given by, 
	\begin{equation}
    \label{eqn:hamming}
	H(n) = \sum_{i = 1}^{n} |(i,j)| \oplus |(i,j^{'})|,
	\end{equation} 
where $n$ is the grid number of the sampled block and $\oplus$ denotes the $\mod$ 2 operations.
\end{enumerate}
Intuitively the Hamming distance $H(n)$ denotes the positional changes of the Betti number $\beta$ after the finite sequence of operations. We now formally define the symmetric distribution of Betti numbers in our grid diagram context as the change positional changes of Betti number in small grid squares as shown Figure~\ref{Ilustration Hamming}. The Figure~\ref{Ilustration Hamming} shows the Hamming distance after the applications of operations. The original grid block containing a particular configurations $\sigma$ of Betti number $\beta$ changes the number of  positional distribution after the application of $T_{2}$ as $12$. So the hamming distance between the two configurations is $12$. After the operation of $T_{2}$ again the Hamming distance decreases to $6$, then after application of the operation $T_{1}$ the hamming distance increases to $14$ and lastly after the operation of $T_{1}$ again the Hamming distance falls back to $6$. The oscillating nature of Hamming distance $H$ is upper bounded proved in our paper in the Section~\ref{Ising}. \newline

\begin{definition}[Symmetric Distribution]
The probability distribution over Betti numbers $\beta$ on a local grid block of size $n$ is symmetric, if the Hamming distance  $H_{i,j}$ computed over the grid block (Equation~\eqref{eqn:hamming}) is bounded by $\eta$ after a finite sequence of operations $T$, 
\begin{equation}
H(n) \leq \eta(n,l)
\end{equation}
\end{definition}
where $\eta(n,l)$ is a integer parameter of choice and it depends on grid number of the sampled block and the grid parameter $l$.\newline 

The conditional probability $\Pr (H \leq \eta)$ is computed as 

\begin{equation}
\Pr (H \leq \eta) = \Pr( \sigma |H \leq \eta)\Pr(\sigma)
\end{equation}

For the case of Cyclic Permutation operation, it is intuitive to see the local symmetry axis passing through the middle of the sampled grid block if the $H(n) = 0$. The value of $H(n)$ gives us the sense of symmetry. More the value of $H(n)$ less will be the symmetry of the sampled grid block. Next, we prove the upper bound of the Hamming distance $H(n)$ is bounded. 

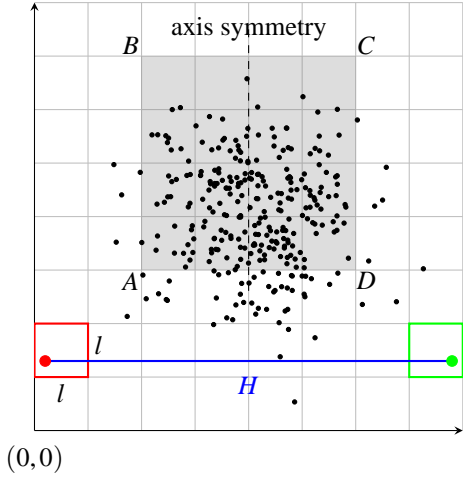
\begin{figure}
\centering
\begin{tikzpicture}
\begin{axis}[
    axis equal image,
    axis lines = left,
    xmin=-4,xmax=4,
    ymin=-4, ymax=4,
    enlargelimits=false,
    grid=both,
	xtick={-3, -2, -1,  0,  1,  2, 3, 4},
    xticklabel=\empty,
    xtick style={draw=none},
    ytick={-3,  -2, -1,  0,  1,  2, 3,4},
    yticklabel=\empty,
    ytick style={draw=none},
	extra x ticks={-4, -3},
	extra x tick labels={{$(0,0)$}, \empty},
    extra x tick style={grid=none},
]

\addplot [only marks, samples=300, mark size=0.8pt] ({invgauss(rnd,rnd)},{invgauss(rnd,rnd)});
\draw [red, thick] (axis cs:-4,-3) rectangle (axis cs:-3,-2);
\draw [green, thick] (axis cs:3,-3) rectangle (axis cs:4,-2);
\addplot[smooth, blue, thick]plot coordinates 
            {
                (-3.8,-2.7)
                (3.8, -2.7)
            }  
            node[above] at (axis cs:0,-3.5) {$H$};
\addplot[mark=*,red] plot coordinates
            {
                (-3.8,-2.7)
            };
\addplot[mark=*,green] plot coordinates 
            {
                (3.8, -2.7)
            };
\node at (axis cs:-3.5,-3.3) {$l$};
\node at (axis cs:-2.8,-2.4) {$l$};
\node at (axis cs: 0, 3.5) {axis symmetry};
\fill[color=gray, opacity=0.27] (axis cs:-2,-1) rectangle (axis cs:2,3);
\addplot[dashed]plot coordinates 
            {
                (0,3.5)
                (0,-1.5)
            };
\node at (axis cs:-2.2, -1.2) {$A$};
\node at (axis cs:-2.2, 3.2) {$B$};
\node at (axis cs: 2.2, 3.2) {$C$};
\node at (axis cs: 2.2, -1.2) {$D$};
\end{axis}
\end{tikzpicture}
\caption{A distribution of Betti numbers on grid}
\label{Hamming}
\end{figure}

For our proof of bounded upper bound of Hamming distance, we assume some known discrete distribution of Betti number $\beta$  on the small grid square i.e the probability distribution $f_{\beta}(i,j)$ of $\beta$ on the grid square $(i,j)$ follows some distributions. This approximation is valid as the real world data lies in between truly random distribution and a probability distribution.  

To simplify the notations, we denote $\mathbf{x} = \{(i,j), l \}$ and the probability density function $f_{\beta} = f_{\beta}(\mathbf{x})$. From hence onward we also denote $f_{\beta}(\mathbf{x})$ as $f(\mathbf{x})$. Thus $f_{\beta}(\mathbf{x})$ defines the probability that  the grid square $(i,j)$ will be occupied by $\beta$. Note that, $f_{\beta}(\mathbf{x})$ is a multivariate function \cite{holtz2008sparse,Garcke:2014:SGA:2669172}.\newline
In the proof, we have used $2$ dimensions as $(1)$ position of the grid square $(2)$ size of the small grid square $l$.

\section{The Ising Model on a 2-D Grid}\label{Ising}
We formally propose a new Ising model of Data and compute the statistical distribution of the Betti numbers. Our modeling of data on a $2{-}D$ grid surface and digitization of Betti numbers are analogous to the spin configuration as in the Quantum Ising Model.  We draw those parallels from the physical Ising model \cite{chakrabarti1996quantum,grimmett2010probability} and propose an analogous \textit{Data Ising Model}. We then compute a probabilistic distribution of configurations Betti numbers on the grid.  This allows us to find the symmetric distribution for the sampled grid after the finite sequence of operations of   commutation, cyclic permutation and stabilization. We introduce the following notations.\newline

We denote the $2{-}D$ planar grid $G = (g, N)$, where $g$ is the small grid square and $N$ is the set of neighboring grid squares as shown in Figure~\ref{Ising Model}. The  small red square grid is surrounded by the set $N$ of four green squares grid as shown in Figure~\ref{Ising Model}. To each small square grid $g \in G$,  we associate a number $1$ or $0$ analogous to quantum spin as in Quantum Ising Model with the local two-dimensional Hilbert space $\mathbb{C}^{2}$ \cite{grimmett2010probability,chakrabarti1996quantum}. We associate each small grid square $g$ with $1$ if the Betti number computed on it is not equal to $0$, and leave the grid square $g$ vacant(as in our model) or mark it with $0$. This marking leaves us the planar grid as a block spin configuration. This allows us to write the \textit{configuration space} for our planar grid as the tensor product of the grid states $(1)$ and $0$ as explained before.\newline
\begin{figure}[h!]
\centering
\begin{tikzpicture}[scale=1.00]
\draw[step= 1cm,color=gray] (0,0) grid (3,3);
\node[green] at (+1.50,+2.50) {$1$};
\node[green] at (+0.50,+1.50) {$1$};
\node[green] at (+1.50,+0.50) {$1$};
\node[red] at (+1.50,+1.50) {$1$};
\node[green] at (+2.50,+1.50) {$1$};
\end{tikzpicture} 
\caption{Ising Model}
    \label{Ising Model}
\end{figure}
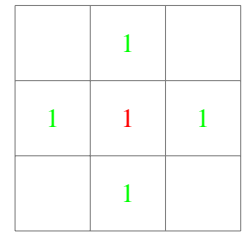\newline
The configuration space $\mathcal{H}$ for the planar $2{-}D$ planar grid is expressed as 
\begin{equation}\label{Tensor Product}
\mathcal{H} = \bigotimes_{g\in G}\mathbb{C}
\end{equation}
for the local Hilbert space $\mathbb{C}$. The eigenvectors for the Hilbert space $\mathbb{C}^{2}$  are $e_{1} = \begin{pmatrix}
           1 & 0
         \end{pmatrix}^T$, $e_{2} = \begin{pmatrix}
           0 & 1
         \end{pmatrix}^T$ of the matrix $ \sigma_{g}^{(3)} =   \begin{pmatrix}
            1  & 0  \\
            0  & -1
         \end{pmatrix}$
at the small grid site $g$, with eigenvalues $\pm 1$.  The other two matrices are  $ \sigma_{g}^{(1)} =   \begin{pmatrix}
            1  & 0  \\
            0  & -1
         \end{pmatrix}$ ,  $ \sigma_{g}^{(2)} =   \begin{pmatrix}
            1  & 0  \\
            0  & -1
         \end{pmatrix}$. We the propose the operator $H$ for our $2{-}D$ planar grid  analogous to the Hamiltonian concept as 
     \begin{equation}\label{sequenceOp}
\mathbb{H} = -\sum_{g_{1},g_{2}\in G}J_{g_{1},g_{2}}\sigma_{g_{1}}^{(3)}\sigma_{g_{2}}^{(3)} -\Gamma\sum_{g\in G}\sigma_{g}^{1}
\end{equation}
where $g_{1}$ and  $g_{2}$ are neighboring small grid squares as shown in Figure~\ref{Ising Model}. \newline

The parameter $J_{g_{1},g_{2}}$ is  defined as the interaction  strength between the small grid squares $g_{1}$ and  $g_{2}$. The parameter $J_{g_{1},g_{2}}$ is critically dependent on the boundaries  $\partial g_{1}$ $\partial g_{2}$ and grid length $l$. The interaction parameter indicates the continuity of data manifold.  The parameter $\Gamma$ in our model denotes the rate of change hierarchical  continuity of across the data manifold.  The is hierarchical scaling variance parameter $\Gamma(l)$ is a function of the dimension of small grid square $l$.

The probability for a configuration $\sigma$ of Betti numbers in our $2{-}D$ planar grid $G$ based on our data based Ising model is  
\begin{equation}\label{Ising Distribution}
P_{G}(\sigma) = \frac{1}{Z_{G}}e^{-\tau  \mathbb{H}(\sigma)}
\end{equation}
where $\tau$ is a parameter. The normalization constant is for all possible configurations $\sigma$ is given by 
\begin{equation}\label{Normalization}
Z_{G} = \sum e^{-\tau \mathbb{H}(\sigma)}
\end{equation}
The expected value for a function $<f(\sigma)>$ of configurations is 
\begin{equation}\label{Expectation}
<f(\sigma)> = \sum_{\sigma}f(\sigma)P_{G}(\sigma)
\end{equation}
We compute the expected value  $H(\sigma)$ of the number of using our proposed Data Ising Model
\begin{equation}\label{Expectation}
<H(\sigma)> =  \sum_{\sigma} H(\sigma) P_{G}(\sigma)
\end{equation}

\begin{remark}[Symmetric Distributions: Trivial Cases]
Given a grid block, when either $P(|(i,j)|)=1$ $\forall i, j$ or $P(|(i,j)|)=0$  $\forall i, j$, the Hamming distance is $0$ after any sequence of legal grid moves.
\end{remark}

\begin{remark}[The Bernoulli Distribution of the Betti Numbers]
Given a grid block of size $n$, when the distribution of the Betti numbers in each grid square independently follows the Bernoulli distribution with a parameter $p$, the expected Hamming Distance after a sequence of grid moves is $2n^2p(1-p)$.   
\end{remark}

\begin{thm}[Bounded Hamming Distance between Symmetric Grid]\label{Proof}
When Commutation, Cyclic Permutation and Stabilization are allowed grid moves in a grid block, the statistical distribution measured by Hamming distance $H(n)$ remains bounded after a sequence of grid moves.
\end{thm}

\begin{proof} 
Let $\Omega \subseteq \mathcal{R}$ be a set and a $d (=2)$ dimensional product measure defined on \textit{Borel} subsets of $\Omega^{2}$ using dimension-wise decomposition approximations as
\begin{equation}
d(\mathbf{x}) = \prod d\mu_{j}(x_{j}) = d\mu_1(x_1) \cdot d\mu_2(x_2),
\end{equation}
where $\mathbf{x} = (x_{1}, x_{2}) $ and $\mu_{j}$ ($j = 1, \cdots, d$) are probability measures on Borel subsets of $\Omega$. Here $x_{1} = (i,j)$, $x_{2} = l$ \\

Let $V^{(2)}$ is the Hilbert space of all functions. We define $f(\mathbf{x})$ as a multivariate density function defined as 
\begin{equation}
f: \Omega^{2} \rightarrow [0,1].
\end{equation}

For a subset $\mathbf{u}\subseteq \mathcal{D}$, where $\mathcal{D} = \{ 1, 2\}$, the measure $\mu$ induces projection functions $P_{\mathbf{u}}: V^{(2)} \rightarrow V^{(|\mathbf{u}|)} $ by
\begin{equation}
 P_{\mathbf{u}}f(\mathbf{x_{u}}) :=  \int_{\Omega^{d-|\mathbf{|u|}}}f(\mathbf{x})d\mu_{\mathcal{D}\setminus u}(\mathbf{x})
\end{equation}

Here  $\mathbf{x_{u}}$ denotes the $|\mathbf{u}|-$dimensional vector and $d\mu_{\mathcal{D}\setminus \mathbf{u}}(\mathbf{x}) := \prod_{j\notin \mathbf{u} }d\mu_{j}(x_{j})$. \newline

For $\mathbf{u = \emptyset}$ the projection function is given as 

\begin{equation}
P_{\emptyset}f(\mathbf{x_{\emptyset}}) = \int_{\Omega^{2}}f(\mathbf{x})d\mu(\mathbf{x}) =:  A
\end{equation}

$f \in V^{(2)}$ is then decomposed using dimension-wise decomposition and  as  
\begin{equation}
 f(\mathbf{x}) = \sum_{\mathbf{u}\subseteq \mathcal{D}}f_{\mathbf{u}}(\mathbf{x_{\mathbf{u}}})
 \end{equation}
 
 with the orthogonality conditions 
\begin{equation}\label{ortho}
(f_{\mathbf{u}},f_{\mathbf{u}}) = \mathbf{0}, \quad \mathbf{u} \neq \mathbf{v}
\end{equation}

The $f_{\mathbf{u}}$ are computed recursively as   
\begin{equation}
f_{\mathbf{u}}(\mathbf{x_{\mathbf{u}}}) = P_{\mathbf{u}}f(\mathbf{x_{\mathbf{u}}}) - \sum_{\mathbf{v} \subset \mathbf{u}}f_{\mathbf{v}}(\mathbf{x_{\mathbf{v}}})
\end{equation}

Using the classical \textit{ANOVA Decomposition} and orthogonality condition we write the \textit{variance}  $\sigma(f)^{2}$ as 
 \begin{equation}
 \begin{split}
\sigma^{2}(f) & = \int_{\Omega^{d}}(f(\mathbf{x}) - A )^{2}d(\mathbf{x}) \\ 
              &  = \sum_{\substack {\mathbf{u}\subseteq \mathcal{D,} \\ {\mathbf{u} \neq \emptyset} }}\sigma^{2}(f_{\mathbf{u}})
\end{split}
\end{equation}
where $\sigma^{2}(f_{\mathbf{u}})$ denotes the variance of $f_{\mathbf{u}}$.\newline

Now we compute the probability of Hamming distance $H = m$ for a grid block consisting of $n^{2}$  grid squares as  
\begin{equation}
\begin{split}
\Pr[H ] & = \Pr[\sum_{i = 1}^{n} |(i,j)| \oplus |(i,j^{'})|]\\
           &  \leq \sum_{i = 1}^{n}\Pr[|(i,j)| \oplus |(i,j^{'})|]
\end{split}
 \end{equation}
 Now for the case $m > n^2$,  \quad $\Pr[H = m ]  = 0$. \\
To get the tighter upper bound we use the transformation of random variables and write as there exists a map $g$ and $g_{1}^{-1}$ as 

\begin{equation}
H = g_{1}(X) 
 \end{equation}
\begin{equation}
X =    g_{1}^{-1}(H) 
 \end{equation}

We prove the the upper bound of Hamming distance after a finite sequence of  \textit{Chebyshev's inequality} we write 
\begin{equation}
\begin{split}
 \Pr[H  > k\chi(\sigma) ] & = \Pr[g_{1}(X)  > k\chi(\sigma) ] \\
                          & \leq \frac{1}{k\kappa(\sigma)} 
 \end{split}
\end{equation}
where $\chi(\sigma)$ and $\kappa(\sigma)$  are functions that depend on the variance of $f(\mathbf{x})$
\end{proof}

 We have proposed a general Algorithm

\begin{algorithm}[h!]
\caption{Searching Local Symmetry}\label{losym}
\begin{algorithmic}[1]
\Require Marked Grid Diagram
\Procedure{Hamming Distance}{Sampled Grid Block}
\State Call \textbf{Generating Grid Diagram}
\State Sample the grid diagram $G$
\State Call $T_{1}$, $T_{2}$, $T_{3}$ to generate $T(G) = T_{1}^{a}\circ T_{2}^{b} \circ T_{3}^{c} $
\State Apply $T(G)$ on the sampled Grid Block
\State Compute $H(n) = \sum_{i = 1}^{n} |(i,j)| \oplus |(i,j^{'})|$
\If{$H(n) \leq \eta $}
    \State the sampled grid is symmetric
    \Else{ the sampled grid is not symmetric}
    \EndIf
\EndProcedure%

\Algphase{Generating Grid Diagram}\label{GHD}
\Require Marking Each Small Grid Square with $\beta$
\Function{Computing}{Betti Number $\beta$}
\State Construct the Grid $G$ of parameters $n$ and $l$
\State $G(i,j)$ is the position of small grid square at $(i,j)$ position.
\State Construct Simplicial Complex for each $G(i,j)$
\State Compute the quotient space $H_{k}(X) = \frac{ker \partial_{k}}{ker \partial_{k+1}}$
\State  $\beta$ = dim $(H_{k}(X))$ 
\State  Mark $G(i,j)$ with $\beta$ if $\beta\neq 0$
\State  Leave $G(i,j)$e empty if $\beta\neq 0$
\EndFunction
\Algphase{Commutation Operation $T_{1}$ }
\Procedure{Commutation}{G} \Comment{$T_1$}
\State $c \gets 1$
\While{$c \leq n - 1$} 
\State swap column $c$ and $c + 1$
\State $c \gets c + 2$
\EndWhile
\EndProcedure
\Algphase{Cyclic Permutation $T_{2}$}

\Procedure{Cyclic Permutation}{G} \Comment{$T_2$}
\For{$c \in 1 \dotsc n$}
	\For{$r \in 1 \dotsc n$}
    	\If{$c = 1$}
			{ $G(c, r) = G(n,r)$}
        \Else { $G(c, r) = G(c - 1,r)$}
        \EndIf
    \EndFor
\EndFor
 \EndProcedure

\Algphase{Stabilization $T_{3}$ }
\Procedure{Stabilization}{G} \Comment{$T_3$}
\State Randomly pick a column $c$ and split it into two
\State insert an empty row and fill the intersections with two columns with $\beta$
\EndProcedure
\end{algorithmic}
\end{algorithm}

\section{Related Work}\label{RWW}
The symmetric features of the data set like rotation symmetry, translation symmetry are used as a feature and used a priory in Bayesian machine learning \cite{citeulike:13098859} or used in training the convolutional neural network layers \cite{2016arXiv160202660D,P115.045}. Analogous to our definitions of symmetric operations of cyclic permutation, commutation and stabilization,    \cite{2016arXiv160202660D} proposes four operations which is inserted  in neural network model as layers to model the translation equivariance into rotation equivariance. The notion of equivariance is formally defined as 

\begin{definition}[Equivariant Function] 
The function $f$ is defined as equivariant for a class of transformations $\mathcal{T}$, if for all transformations $\mathbf{T}\in \mathcal{T}$ of the input $x$,  there exists a corresponding transformation $\mathbf{T}^{'}$ of the output $f(x)$, such that the following condition holds
\begin{equation}\label{Equivariant Equation}
f(\mathbf{T}(x)) = \mathbf{T}^{'}f(x)
\end{equation}
\end{definition}

The patterns at different spatial positions are encoded similarly in the feature representations by these layers. This allows parameter sharing much more effectively than a fully connected neural network under similar conditions. They extended to rotation invariance by introducing the four operations of $(1)$ \textbf{Slice} $(2)$   \textbf{ Roll}   $(3)$ \textbf{ Pool }  $(4)$ \textbf{Stack} to build CNNs. The CNNs will detect the cyclic symmetry  in the input data by the rotation over the angles $k.90^{o}, k\in \{ 0, 1, 2, 3\}$. They this group of four rotations form a cyclic group of order $4(C_{4})$ as a restricted form rotational symmetry called \textit{cyclic symmetry}. Similarly the \textit{dihedral symmetry}$D_{4}$ is defined as a set of total eight possible orientations after the operation of horizontal flipping. \cite{2016arXiv160202660D} proposes the computation of approximate invariance by the method of \textit{data augmentation} as presenting the network during training with examples that are randomly perturbed. Given a network with sufficient capacity, it learn invariances.

\section{Experimental Results}\label{ERI}
In this section, we setup a grid block with size $1000 \times 1000$. We conduct two scenarios of sampling Betti number. In the first case, Betti number positions are sampled inside the grid with by a mixtures two Gaussian distributions $\mathcal{N}(\mu_1, \Sigma_1)$ and $\mathcal{N}(\mu_2, \Sigma_2)$. A 2-dimensional Gamma distribution $\Gamma(k,\theta) $ is chosen to generate Betti number position in the second case. The grid block of sampled Betti number position is divided into subsample grid squares with size $5 \times 5$. We perform grid moves including commutation, cyclic permutation, stabilization on these local grid squares. After the transformations, the Hamming distances are obtained between the original grid squares and the corresponding transformed grid squares. With this synthetic data, we conduct four types of tests including (a) commutation, (b) cyclic permutation, (c) chain of transformations, and (d) chains of transformation with noise data. The contour line illustrations of results are portrayed in Figure \ref{fig:gauss} and ~\ref{fig:gamma} respectively for the mixture of Gaussian case and Gamma case.
 \begin{figure}[t]
    \centering
\includegraphics[width=0.5\textwidth]{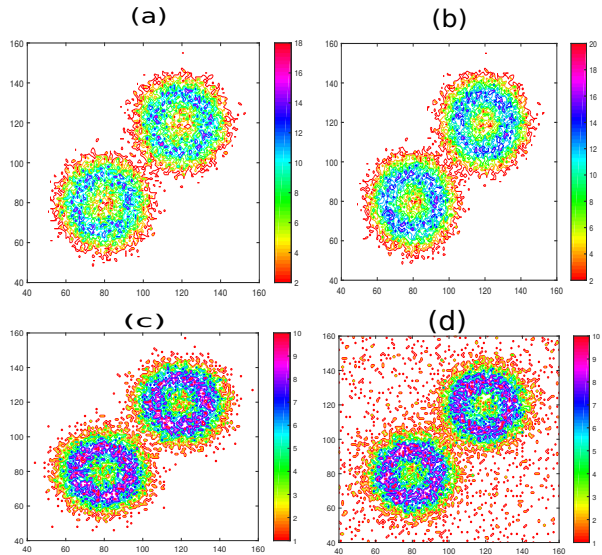}
\caption{Mixture of Gaussian distribution case: (a) Hamming distance between original and commutation; (b) Hamming distance between original and cyclic permutation (c); Hamming distance between the original grid and the grid after chains of transformation $T = T_{1} \circ T_{2} \circ T_{1} \circ T_{2} \circ T_{1} \circ T_{2} \circ T_{1} \circ T_{2} \circ T_{1} \circ T_{2}$; (d) Hamming distance between the original grid and the grid after chains of transformation $T = T_{1} \circ T_{2} \circ T_{1} \circ T_{2} \circ T_{1} \circ T_{2} \circ T_{1} \circ T_{2} \circ T_{1} \circ T_{2}$ with noise}
\label{fig:gauss}
\end{figure}

\begin{figure}[!h]
\includegraphics[width=0.5\textwidth]{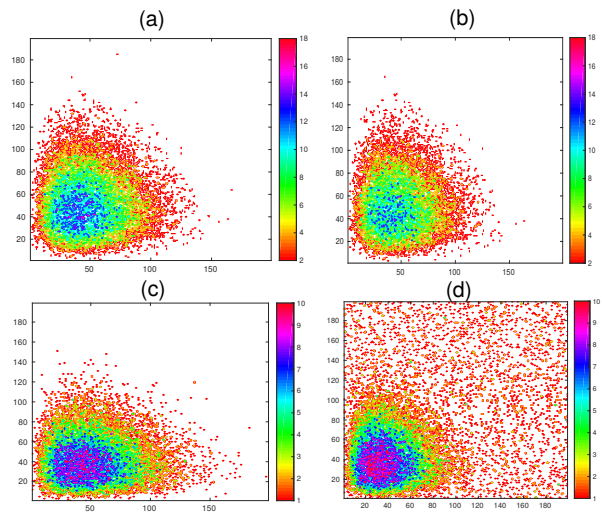}
\caption{Gamma distributions case: (a) Hamming distance between original and commutation; (b) Hamming distance between original and cyclic permutation (c); Hamming distance between the original grid and the grid after chains of transformation $T = T_{1} \circ T_{2} \circ T_{1} \circ T_{2} \circ T_{1} \circ T_{2} \circ T_{1} \circ T_{2} \circ T_{1} \circ T_{2}$; (d) Hamming distance between the original grid and the grid after chains of transformation $T = T_{1} \circ T_{2} \circ T_{1} \circ T_{2} \circ T_{1} \circ T_{2} \circ T_{1} \circ T_{2} \circ T_{1} \circ T_{2}$ with noise}
\label{fig:gamma}
\end{figure}

\section{Conclusions}\label{Conc}
We have proposed a novel method of finding symmetry termed as \textit{grid symmetry } in data  by developing a new framework of $2-$D grid space. We have proposed  three fundamental operations of commutation, cyclic permutation and stabilization to determine the symmetry. The methods of statistical  topology i.e distribution of Betti number is used as a feature in checking symmetry in data. Our method is particularly helpful  Bayesian machine learning where the topological feature(Betti number) is encoded a priory. Our method of spatial distribution of Betti numbers on  $2D$ grid can be encoded in constitutional neural network layers as the property of \textit{translation equivariance} \cite{2016arXiv160202660D,P115.045}.  The method of \textit{data augmentation} as described in \cite{2016arXiv160202660D} for the training of CNN's fits particularly well with our approach, as the random perturbations are well described the topological deformations. Our method throws light on the directions of studying the deep machine learning using scale invariants. We have connected our low dimensional topology models with Ising.  Modeling invariances in deep learning particularly so in unsupervised learning is an active area of research\cite{DBLP:journals/corr/SrivastavaMS15}. The recent Google's breakthrough cat neuron paper the authors uses the unshared weights to allow learning of more invariances other than translational invariances\cite{Le_buildinghigh-level}. Our modeling of topological invariances(Betti number) and priors in the data set naturally fits in the scheme. Further our Ising model parameter $\Gamma(l)$ captures scaling of invariants asymptotically, as we decrease the dimension of the small grid square $l$.

\bibliographystyle{named}
\bibliography{RKL}
\end{document}